\documentclass{article}

\usepackage{relsize}
\usepackage{wrapfig}
\usepackage{balance}
\usepackage{array}
\usepackage{amsmath}
\usepackage{mathrsfs}
\usepackage{amssymb}
\usepackage{amsthm}
\usepackage{dsfont}
\usepackage{grffile}
\usepackage{lmodern}
\usepackage[disable]{todonotes}

\newtheorem{lemma}{Lemma}
\newtheorem{theorem}{Theorem}
\newtheorem{definition}{Definition}

\newtheorem{remark}{Remark}

\renewcommand{\P}{{\mathbb P}}
\newcommand{\R}{{\mathbb R}}

\newcommand{\transpose}{^\mathsf{\scriptscriptstyle T}}

\newcommand{\cD}{\mathcal{D}}

\newcommand{\cF}{\mathcal{F}}
\newcommand{\cG}{\mathcal{G}}

\newcommand{\cL}{\mathcal{L}}

\newcommand{\cO}{\mathcal{O}}

\newcommand{\cV}{\mathcal{V}}
\newcommand{\cW}{\mathcal{W}}

\newcommand{\be}{{\bf e}}

\newcommand{\br}{{\bf r}}

\newcommand{\bq}{{\bf q}}
\newcommand{\bu}{{\bf u}}
\newcommand{\by}{{\bf y}}
\newcommand{\bx}{{\bf x}}
\newcommand{\bA}{{\bf A}}
\newcommand{\bD}{{\bf D}}
\newcommand{\bM}{{\bf M}}
\newcommand{\bI}{{\bf I}}
\newcommand{\bQ}{{\bf Q}}
\newcommand{\bX}{{\bf X}}
\newcommand{\bU}{{\bf U}}
\newcommand{\bV}{{\bf V}}

\newcommand{\balpha}{{\boldsymbol \alpha}}

\newcommand{\bLambda}{{\boldsymbol \Lambda}}
\newcommand{\bmu}{{\boldsymbol \mu}}
\newcommand{\bxi}{{\boldsymbol \xi}}

\DeclareMathOperator*{\argmax}{arg\,max}
\DeclareMathOperator*{\argmin}{arg\,min}

\newcommand{\beq}{\begin{equation}}
\newcommand{\eeq}{\end{equation}}

\newcommand{\beqa}{\begin{eqnarray}}
\newcommand{\eeqa}{\end{eqnarray}}

\newcommand{\beqan}{\begin{eqnarray*}}
\newcommand{\eeqan}{\end{eqnarray*}}

\newcommand{\beqal}{\begin{align*}}
\newcommand{\eeqal}{\end{align*}}

\newcommand{\SpectralUCB}{{\scshape SpectralUCB}}
\newcommand{\LinearEliminator}{{\scshape LinearEliminator}}
\newcommand{\SpectralEliminator}{{\scshape SpectralEliminator}}

\usepackage{times}
\usepackage{graphicx} 
\usepackage{subfigure}
\usepackage{natbib}

\usepackage{algorithm}
\usepackage{algorithmic}

\usepackage{hyperref}

\usepackage[accepted]{icml2014} 

\icmltitlerunning{Spectral Bandits}

\begin{document} 

\twocolumn[
\icmltitle{Spectral Bandits for Smooth Graph Functions}

 \vspace*{-1em}
\icmlauthor{Michal Valko}{michal.valko@inria.fr}
 \vspace*{-0.1em}
  \icmladdress{INRIA Lille - Nord Europe, SequeL team, 40 avenue Halley 59650,
 Villeneuve d'Ascq, France}
  \vspace*{-0.4em}

\icmlauthor{R\' emi Munos}{remi.munos@inria.fr}
 \vspace*{-0.1em}
\icmladdress{INRIA Lille - Nord Europe, SequeL team, France; Microsoft Research New England, Cambridge, MA, USA}

 \vspace*{-0.4em}

\icmlauthor{Branislav Kveton}{branislav.kveton@technicolor.com}
 \vspace*{-0.1em}
\icmladdress{Technicolor Research Center, 735 Emerson St, Palo Alto, CA
94301, USA}
 \vspace*{-0.4em}

\icmlauthor{Tom\'a\v s Koc\'ak}{tomas.kocak@inria.fr}
 \vspace*{-0.1em}
\icmladdress{INRIA Lille - Nord Europe, SequeL team, 40 avenue Halley 59650,
Villeneuve d'Ascq, France}
 \vspace*{-0.4em}


\icmlkeywords{spectral, bandits, graphs, recommender systems, cumulative regret}

\vskip 0.3in
]

\vspace{-1.5em}

\begin{abstract}
Smooth functions on graphs have wide applications in manifold and
semi-supervised learning. In this paper, we study a bandit problem where the
payoffs of arms are smooth on a graph. This framework is suitable for solving
online learning problems that involve graphs, such as content-based
recommendation. In this problem, each item we can recommend is a node and its
expected rating is similar to its neighbors. The goal is to recommend items that
have high expected ratings. We aim for the algorithms where the cumulative
regret with respect to the optimal policy would not scale poorly with the number
of nodes. In particular, we
introduce the notion of an \emph{effective dimension}, which is small in
real-world graphs, and propose two algorithms for solving our problem that scale
linearly and sublinearly in this dimension. Our experiments on real-world
content recommendation problem show that a good estimator of user preferences
for thousands of items can be learned from just  tens of node evaluations.
\end{abstract}

\section{Introduction}
\label{sec:intro}

\emph{A smooth graph function} is a function on a graph that returns similar
values on neighboring nodes. This concept arises frequently in manifold and
semi-supervised learning \cite{zhu2008semi-supervised}, and reflects the fact
that the outcomes on the neighboring nodes tend to be similar. It is well-known
\cite{belkin2006manifold,belkin2004regularization} that a smooth graph function
can be expressed as a linear combination of the eigenvectors of the graph
Laplacian with smallest eigenvalues. Therefore, the problem of learning such 
function can be cast as a regression problem on these eigenvectors. This is
the first work that brings this concept to bandits. In particular, we study a
bandit problem where the arms are the nodes of a graph and the expected payoff
of pulling an arm is a smooth function on this graph.

\begin{figure}[t]
 \vskip -0.5em
  \begin{center}
   \includegraphics[viewport = 112 348 497 558,clip,width=0.9\columnwidth]
 {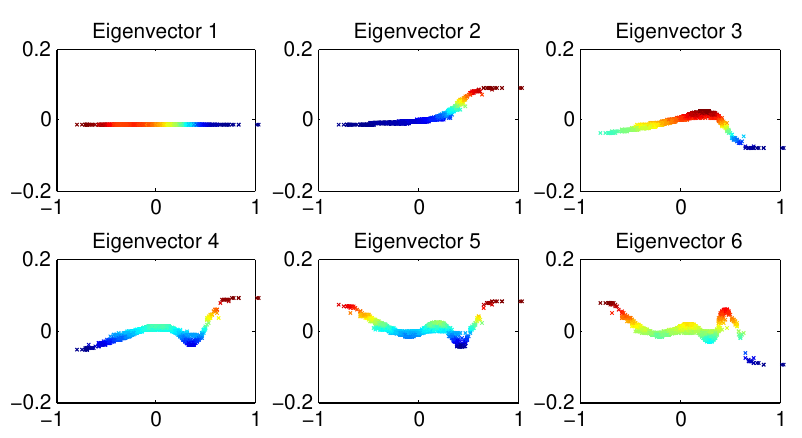}
 \vskip -1em
 \caption{Eigenvectors from the Flixster data corresponding to the smallest
few eigenvalues of the graph Laplacian projected onto the first principal
component of data. Colors indicate the values.}
 \vskip -1.7em
  \label{fig:flixster_eigenvectors}
  \end{center}
\end{figure}

We are motivated by a range of practical problems that involve graphs. One
application is \textit{targeted advertisement} in social networks. Here, the
graph is a social network and our goal is to discover a part
of the network that is interested in a given product. Interests of people in a
social network tend to change smoothly \cite{mcpherson2001birds}, because
friends tend to have similar preferences. Therefore, we take advantage of
this structure and formulate this problem as learning a smooth preference
function on a graph.

Another application of our work are \textit{recommender systems}
\cite{jannach2010recommender}. In content-based recommendation
\cite{chau2011apolo}, the user is recommended items that are
similar to the
items that the user rated highly in the past. The assumption is that users
prefer similar items similarly. The similarity of the items can be measured for
instance by a nearest neighbor graph \cite{billsus2000learning}, where
each item is a node and its neighbors are the most similar items.

In this paper, we consider the following learning setting. The graph is known in
advance and its edges represent the similarity of the nodes. At
time $t$, we choose a node and then observe its payoff. In targeted
advertisement, this may correspond to showing an ad and then observing whether
the person clicked on the ad. In content-based recommendation, this may
correspond to recommending an item and then observing the assigned rating. Based
on the payoff, we update our model of the world and then the game proceeds into
time $t + 1$. Since the number of nodes $N$ can be huge, we are interested in
the regime when $t < N$.

If the smooth graph function can be expressed as a linear combination of $k$
eigenvectors of the graph Laplacian, and $k$ is small and known, our learning
problem can be solved using ordinary linear bandits
\cite{auer2002using,li2010contextual}. In practice, $k$ is problem specific and
unknown. Moreover, the number of features $k$ may approach the number of nodes
$N$. Therefore, proper regularization is necessary, so that the regret of the
learning algorithm does not scale with $N$. We are interested in the setting
where the regret is independent of $N$ and therefore this problem is
non-trivial.


\vspace{-0.4em}
\section{Setting}

Let $\cG$ be the given graph with the set of nodes $\cV$ and
denote $|\cV| = N$ the number of nodes.
Let $\cW$ be the $N \times N$ matrix of similarities $w_{ij}$ (edge
weights) and $\cD$ is the $N \times N$ diagonal matrix with the entries $d_{ii}
= \sum_j w_{ij}$.
The graph Laplacian of $\cG$ is defined as $\cL = \cD - \cW$.
Let $\{\lambda^\cL_k, \bq_k\}_{k=1}^N$ be the eigenvalues and eigenvectors
of $\cL$ ordered such that $ 0 = \lambda^\cL_1 \le \lambda^\cL_2 \le \dots
\le \lambda^\cL_N$.
Equivalently, let $\cL = \bQ \bLambda_\cL \bQ \transpose$ be the
eigendecomposition
of $\cL$, where $\bQ$ is an $N \times N$ \textit{orthogonal} matrix with
eigenvectors in columns.

In our setting we assume that the reward function  is a linear
combination of the eigenvectors. For any set of weights $\balpha$ let
$f_\balpha: \cV \to \R$ be the function defined on nodes,
linear in the basis of the eigenvectors of $\cL$:
\[
 f_\balpha(v) = \sum_{k=1}^{N} \alpha_k (\bq_{k})_v = \langle \bx_v , \balpha
\rangle,
\]
where $\bx_v$ is the $v$-th row of $\bQ$, i.e., $(\bx_{v})_i = (\bq_{i})_v$.
If the weight coefficients of the true $\balpha^*$
are such that the large coefficients correspond
to the eigenvectors with the small eigenvalues
and vice versa, then $f_{\balpha^*}$ would be a smooth function on $\cG$
\cite{belkin2006manifold}.
Figure~\ref{fig:flixster_eigenvectors}
displays first few eigenvectors of the Laplacian
constructed from data we use in our experiments.
In the extreme case, the true $\balpha^*$ may be of the form
$[\alpha_1^*,\alpha_2^*, \dots, \alpha_k^*, 0, 0, 0]\transpose_N$
for some $k\ll N$. Had we known $k$ in such case, the
known linear bandits algorithm would work with the performance
scaling with $k$ instead of $D=N$. Unfortunately, first, we do not know $k$
and second, we do not want to assume such an extreme case (i.e.,~$\alpha_i^* =
0$ for $i>k$). Therefore, we opt for the more plausible assumption
that the coefficients with the high indexes are small. Consequently, we deliver
algorithms with performance that scales with the
smoothness with respect to the graph.


%

The learning setting is the following. In each time step $t\le T$,
the recommender $\pi$ chooses a node $\pi(t)$ and obtains
a noisy reward such that:
\[
  r_t =\langle \bx_{\pi(t)} , \balpha^* \rangle + \varepsilon_t,
\]
where the noise $\varepsilon_t$ is assumed to be $R$-sub-Gaussian for any $t$.
In our setting, we have $\bx_v \in \R^D$ and $\|\bx_v\|_2\leq 1$
for all $\bx_v$.  The goal of the recommender is to minimize the cumulative
regret with respect to the strategy that always picks the best node
w.r.t.~$\balpha^*$.
Let $\pi(t)$ be the node picked (referred to as 
\textit{pulling an arm}) by an algorithm $\pi$ at time $t$.
The cumulative (pseudo) regret of $\pi$ is defined as:
\[
R_T  = T \max_v  f_{\balpha^*}(v) -  \sum_{t=1}^T f_{\balpha^*}(\pi(t))
\]
We call this bandit setting \textit{spectral}
since it is built on the spectral properties of a graph.
Compared to the linear and multi-arm bandits,
the number of arms
 $K$ is equal to the number of nodes $N$ and to the
 dimension of the basis $D$  (eigenvectors are of dimension $N$).
However, a regret that scales with $N$ or $D$
that can be obtained using those settings is not acceptable because the number
of nodes can be large.
While we are mostly interested in the setting with $K=N$,
our algorithms and analyses can be applied for any finite~$K$.

\vspace{-0.4em}
\section{Related Work}
\label{sec:related}

Most related setting to our work is that of the linear
and contextual linear bandits. 
\citet{auer2002using} proposed a
SupLinRel algorithm and showed that it obtains $\sqrt{DT}$
regret that matches the lower bound by \citet{dani2008stochastic}.
First practical and empirically successful
algorithm was LinUCB \cite{li2010contextual}.
Later, \citet{chu2011contextual} analyzed 
SupLinUCB, which is a LinUCB equivalent of SupLinRel, to show that
it also obtains $\sqrt{DT}$ regret.
\citet{abbasi2011improved}
proposed the OFUL algorithm
for linear bandits
which obtains $D\sqrt{T}$ regret.
Using their analysis, it is possible
to show that LinUCB obtains  $D\sqrt{T}$
regret as well (Remark~\ref{rem:linucb}). Whether LinUCB matches the
$\sqrt{DT}$ lower bound
for this setting is still an open problem.

\citet{abernethy2008competing} and \citet{bubeck2012towards} studied a more
difficult \textit{adversarial} setting of linear bandits where the reward
function is time-dependent. It is an open problem if this
approaches would work in our setting and have
an upper bound on the regret that scales better than with~$D$.

\citet{kleinberg2008multi,slivkins2009contextual}, and \citet{bubeck2011x} use
similarity
information
between the context of arms, assuming a Lipschitz or more general
properties.
While such settings are indeed more general, the regret bounds
scale worse with the relevant dimensions.
\citet{srinivas2009gaussian} and \citet{valko2013finite} also perform maximization
over the smooth functions that are either sampled from a Gaussian process
prior or have a small RKHS norm.  Their setting is also more general
than ours since it already generalizes linear bandits. However 
their regret bound in the linear case also scales with $D$.
Moreover, the regret of these algorithms also
depends on a quantity for which data-independent
bounds exists only for some kernels, while
our effective dimension is always computable
given the graph.

Another bandit graph setting called the \textit{gang of bandits}
was studied by \citet{cesa-bianchi2013gang}
where each node is a linear bandit with its own
weight vector which are assumed to be smooth on the graph.
Next, \citet{caron2012leveraging} assume that they obtain the
reward not only from the selected node
but also from all its neighbors.
Yet another kind of the partial observability model for bandits on graphs
 in the adversarial setting is considered by \citet{alon2013from}.

\vspace{-0.4em}
\section{Spectral Bandits}
\label{sec:spectralbandits}

In our setting, the arms are \textit{orthogonal} to each other.
Thinking that the reward observed
 for an arm does not provide any information for other arms
would not be correct because of the assumption that under another
basis, the unknown parameter has a low norm.
This provides additional information across the arms through the estimation
of the parameter. We could think of our setting as an $N$--armed bandit
problem where $N$ is larger than the time horizon $T$  and the
mean reward $\mu_k$ for each arm $k$ satisfies the property that under a
change of coordinates, the vector has small weights, i.e.,~there exists a known
orthogonal matrix $\bU$ such that $\balpha = \bU \bmu$ has a low norm.
As a consequence, we can estimate $\balpha$ using penalization
and then recover $\bmu$.

 Given a vector of weights
$\balpha$,
we define its $\bLambda$ norm as:
\begin{align}
\label{eq:spectalpenalty}
  \|\balpha\|_{\bLambda} = \sqrt{\sum_{k=1}^{N}  \lambda_k \alpha_k^2} =
\sqrt{\balpha\transpose \bLambda \balpha}
\end{align}

We make use of this norm later in our algorithms.
Intuitively, we would like to penalize the coefficients $\balpha$
that correspond to the eigenvectors with the large eigenvalues, in other
words, to the less smooth basis functions on $\cG$.

\vspace{-0.4em}
\subsection{Effective dimension}
\label{ssec:introeffd}

In order to present our algorithms and analyses,
we introduce a notion of \textit{effective dimension} $d$. We keep using capital
$D$ to denote the ambient dimension, which is equal to $N$ in the spectral
 setting. Intuitively, the effective dimension is a proxy
for the number of relevant dimensions. We first provide a
formal definition and then discuss its properties.

In general, we assume there exists a diagonal matrix $\bLambda$ with the
entries $0<\lambda=\lambda_1\leq
\lambda_2\leq \dots\leq \lambda_N$ and
a set of $K$ vectors $\bx_1,\dots, \bx_K\in\R^N$ such that $\|\bx_i\|_2\leq 1$
for all $i$. For the spectral bandits, we have $K = N$.
Moreover, since $\bQ$ is an orthonormal matrix, $\|\bx_i\|_2 = 1$.
Finally, since the first eigenvalue of a graph Laplacian is always zero,
 $\lambda^\cL_1 = 0$, we use $\bLambda = \bLambda_{\cL} + \lambda \bI$,
in order to have $\lambda_1 = \lambda$.

\begin{definition}\label{def:effectived}
 Let the \textbf{effective dimension} $d$ be the largest $d$ such
that:
$$ (d-1) \lambda_d  \leq \frac{T}{\log(1 + T /\lambda)}$$
\end{definition}

The effective dimension $d$ is small when the coefficients $\lambda_i$ grow
rapidly above $T$. This is the case when the dimension of the space $D$ (and
$K$) is much larger than $T$, such as in graphs from social networks with very
large number of nodes $N$. In contrast, when the coefficients are all small
then $d$ may be of the order of $T$, which would make the regret bounds useless.
Figure~\ref{fig:effd} shows how $d$ behaves compared to $D$ on the
generated and the real Flixster network graphs\footnote{We set $\bLambda$ to
$\bLambda_{\cL} + \lambda \bI$ with $\lambda = 0.01$,
where $\bLambda_{\cL}$ is the graph Laplacian of the respective graph.} that we
use for the experiments
in Section~\ref{sec:exp}.

\begin{figure}[ht]
 \begin{center}
\includegraphics[width=0.49\columnwidth]{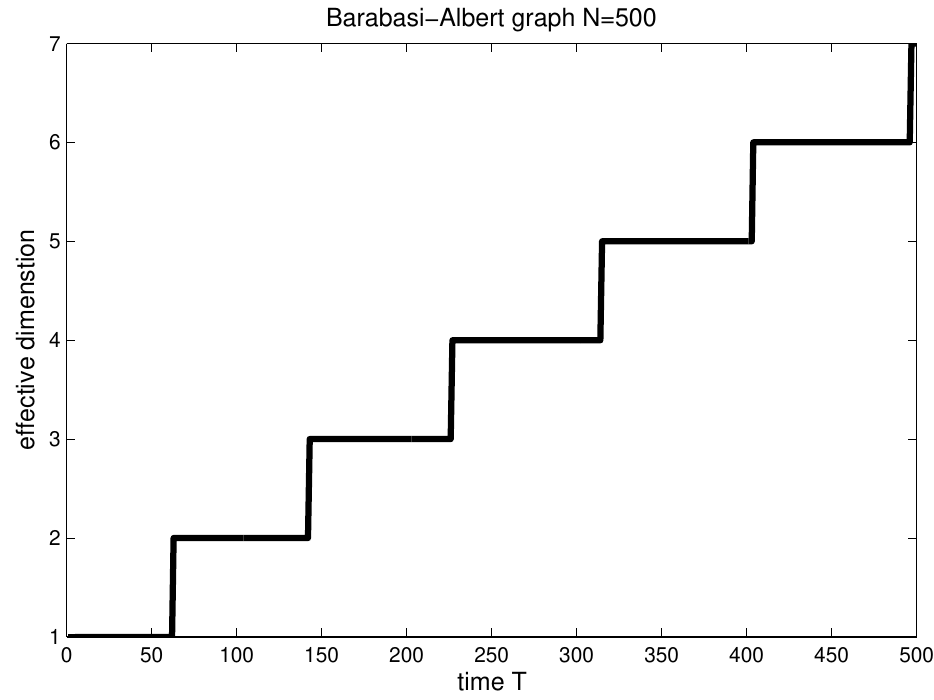}
\includegraphics[width=0.50\columnwidth]{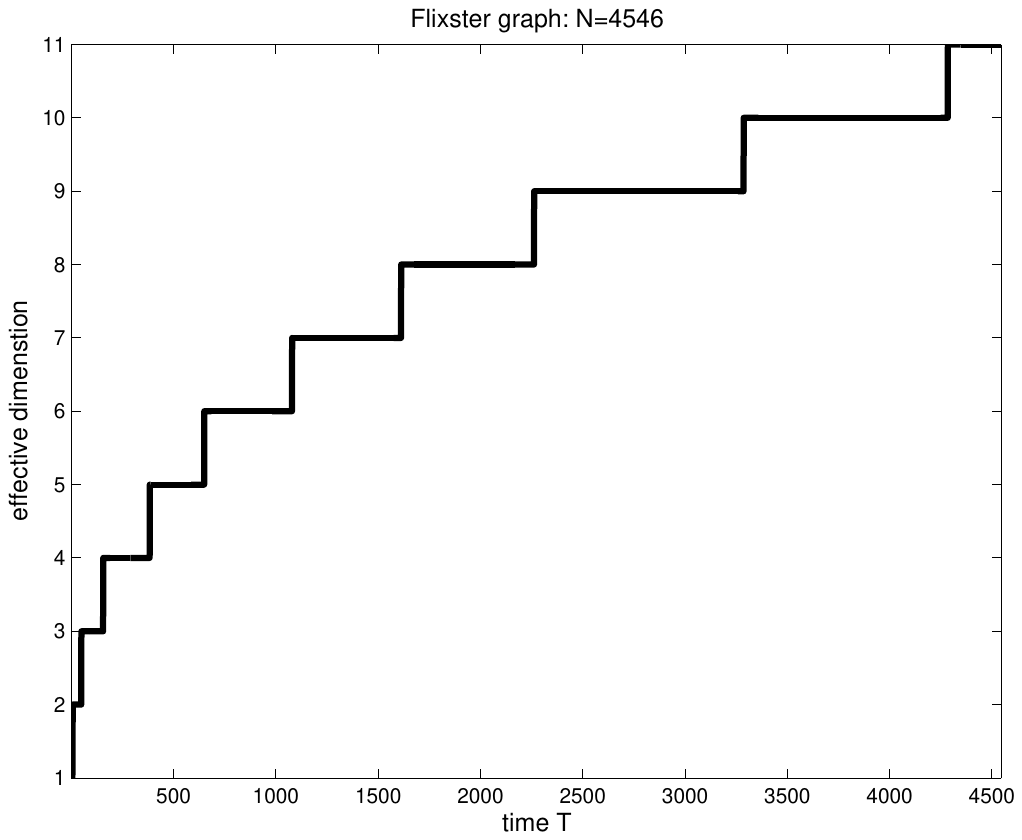}
 \caption{Effective dimension $d$ in the regime $T < N$. The effective dimension
 for this data is much smaller than the ambient dimension $D = N$, which is 500
and 4546 respectively.}
 \label{fig:effd}
 \end{center}
 \end{figure}

The actual form of Definition~\ref{def:effectived} comes from
Lemma~\ref{lem:logdetratio} and will become apparent in 
Section~\ref{sec:analysis}.
The dependence of the effective dimension on $T$ comes from the fact,
that $d$ is related to the number of ``non-negligible'' dimensions
characterizing the space where the solution to the penalized least-squares may
lie, since this solution is basically constrained to an ellipsoid defined
by the inverse of the eigenvalues multiplied by $T$.
Consequently, $d$ is related to the metric dimension of this ellipsoid.
Therefore, when $T$ goes to infinity, the
constraint is relaxed and all directions matter, thus the solution can be
anywhere in a (bounded) space of dimension $N$. On the contrary, for a smaller
$T$, the
ellipsoid possesses a smaller number of ``non-negligible'' dimensions.
Notice that it is natural that this effective dimension depends on $T$ as
we consider the setting $T<N$.
If we
wanted to avoid $T$ in the definition of $d$, we could define it as well in
terms
of $N$ by replacing $T$ by $N$ in Definition~\ref{def:effectived}, but this
would
only loosen its value.

\subsection{\SpectralUCB}
\label{ssec:algospectralucb}

\begin{algorithm}[t]
  \caption{\SpectralUCB}
  \label{alg:TUCB}
\begin{algorithmic}
  \STATE {\bfseries Input:}
  \STATE \quad  $N:$ the number of nodes, $T:$ the number of pulls
  \STATE \quad  $\{\bLambda_{\cL}, \bQ\}$ spectral basis of $\cL$
  \STATE \quad  $\lambda, \delta:$ regularization and confidence parameters
  \STATE \quad  $R,C:$ upper bounds on the noise and $\|\balpha^*\|_\bLambda$
  \STATE {\bfseries Run:}
  \STATE  $\bLambda \gets \bLambda_{\cL} + \lambda \bI$
  \STATE  $d \gets \max\{d : (d - 1) \lambda_d \leq  T/\log(1 + T /\lambda) \} $
  \STATE  $c \gets 2 R\sqrt{ d \log(1 + T/\lambda) + 2\log(1/\delta)} + C $
\FOR{$t = 1$ {\bfseries to} $T$}
  \STATE Update the basis coefficients $\hat\balpha$:
  \STATE \quad $\bX_t \gets [\bx_1,\dots,\bx_{t-1}]\transpose$
  \STATE \quad $\br_t\ \, \gets [r_1,\dots,r_{t-1}]\transpose$
  \STATE \quad $\bV_t \gets \bX_t\transpose\bX_t + \bLambda$
  \STATE \quad $\hat\balpha_t \,\gets \bV_t^{-1}\bX_t\transpose \br_t$
  \STATE Choose the node $v_{t}$ ($\bx_{v_{t}}$-th row of $\bQ$):
  \STATE \quad $v_{t} \gets \argmax_v \left( f_{\hat\balpha_t}(v) + c \|\bx_v \|_{\bV_t^{-1}} \right)$
  \STATE Observe the reward $r_t$
\ENDFOR
\end{algorithmic}
\end{algorithm}
The first algorithm we present is \SpectralUCB\ (Algorithm~\ref{alg:TUCB})
which is based on LinUCB and uses the \textit{spectral
penalty}~\eqref{eq:spectalpenalty}.
For clarity, we set $\bx_{t} = \bx_{v_{t}} = \bx_{\pi(t)}$.
Here we consider regularized
least-squares estimate $\hat\balpha_t$ of the form:
 \vspace{-0.3em}
$$
\hat\balpha_t = \argmin_{\balpha} \left( \sum_{v=1}^{t}\left[\langle \bx_v,
\balpha \rangle - r_v\right]^2 + \|\balpha\|_{\bLambda}^2  \right)
$$
A key part of the algorithm is to define the $c_t\|\bx\|_{\bV_t^{-1}}$
confidence widths for the prediction of the rewards.
We take advantage of our analysis (Section~\ref{sec:effd}) to define $c_t$ based on the effective dimension $d$ which is specifically tailored to our setting.
By doing this we also avoid the computation of the determinant (see Section~\ref{sec:analysis}).
The following theorem characterizes the performance
of \SpectralUCB\/ and bounds the regret as a function of
effective dimension~$d$.

\begin{theorem}
\label{thm:tucb}
Let $d$ be the effective dimension and $\lambda$ be the minimum
eigenvalue of $\bLambda$. If $\| \balpha^* \|_{\bLambda} \leq C$
and for all $\bx_v$, $\langle \bx_v,\balpha^* \rangle \in [-1, 1]$,  then the
cumulative regret of \SpectralUCB\/  is with probability at
least $1-\delta$ bounded as:
\begin{align*}
R_T \leq  &\left[4 R\sqrt{ d \log(1 + T/\lambda) + 2\log(1/\delta)} + 2C +2\right] \\
 & \times \sqrt{4d T \log (1 + T/\lambda)}  
\end{align*}
\end{theorem}

\begin{remark}
The constant $C$ needs to be such that $\| \balpha^* \|_{\bLambda} \leq C$.
If we set $C$ too small, the true $\balpha^*$ will lie outside of the region and far from $\hat\balpha_t$, causing the algorithm to underperform.
Alternatively, $C$ can be time dependent, e.g.,~$C_t = \log T$. In such case, we
do not need to know an upper bound on $\| \balpha^* \|_{\bLambda}$
in advance, but our regret bound would only hold after some
$t$,  when $C_t \geq \| \balpha^* \|_{\bLambda}$.
\end{remark}

We provide the proof of Theorem~\ref{thm:tucb} in Section~\ref{sec:analysis}
and examine the performance of \SpectralUCB\/ experimentally
in Section~\ref{sec:exp}. The $d\sqrt{T}$ result of Theorem~\ref{thm:tucb}
is to be compared with the classical linear bandits,
where LinUCB is the algorithm often used in practice \cite{li2010contextual}
achieving $D\sqrt{T}$ cumulative regret.
As mentioned above and demonstrated in Figure~\ref{fig:effd}, in the $T < N$ regime we can expect $d \ll D = N$
and obtain an improved performance.
\vspace{-0.4em}
\subsection{\SpectralEliminator\/}

\begin{algorithm}[t]
\caption{\SpectralEliminator\/}
 \label{alg:TransductiveElimination}
 \begin{algorithmic}
 \STATE {\bfseries Input:}
 \STATE \quad  $N:$ the number of nodes, $T:$ the number of pulls
\STATE \quad  $\{\bLambda_{\cL}, \bQ\}$ spectral basis of $\cL$
\STATE \quad  $\lambda:$ regularization  parameter
\STATE \quad  $\beta, \{t_j\}_j^J$ parameters of the elimination and phases
 \STATE \quad $A_1 \gets \{\bx_1,\dots,\bx_K\}$.
 \FOR {$j=1$ {\bfseries to} $J$}
 \STATE $\overline{\bV}_{t_j}\gets  \bLambda_{\cL} + \lambda \bI$
 \FOR{ $t=t_{j}$ {\bfseries to} $\min(t_{j+1}-1,T)$}
   \STATE Play $\bx_t\in A_j$ with the largest width to observe 
$r_t$:
   \STATE \quad $\bx_t \gets \arg\max_{\bx \in A_j} \|\bx\|_{\overline{\bV}_{t}^{-1}}$
   \STATE $\overline{\bV}_{t+1} \gets \overline{\bV}_{t}+\bx_t \bx_t\transpose$  
 \ENDFOR
\STATE Eliminate the arms that are not promising:
\STATE \quad $\hat\balpha_{j+1} \gets \overline{\bV}_{t+1}^{-1}
[\bx_{t_{j}},\dots,\bx_{t}]
[r_{t_{j}},\dots,r_{t}]\transpose$
 \STATE \quad  $p \gets \max_{\bx\in A_{j}} \Big[ \langle \hat\balpha_{j+1}, \bx
\rangle - \|\bx\|_{\overline{\bV}_{t+1}^{-1}}\beta \Big]$
 \STATE \quad  $A_{j+1} \gets  \Big\{\bx\in A_{j}, \langle\hat\balpha_{j+1}, \bx
\rangle + \|\bx\|_{\overline{\bV}_{t+1}^{-1}}\beta \geq p  \Big\}$
 \ENDFOR
 \end{algorithmic}
\end{algorithm}

\balance
It is known that the available upper bound for LinUCB or OFUL is not optimal for
the linear bandit
setting with finite number of arms in terms of dimension $D$.  On the other
hand, the algorithms SupLinRel or SupLinUCB
achieve the optimal $\sqrt{DT}$ regret.
In the following, we likewise provide an
algorithm that also scales better with $d$
and achieves $\sqrt{dT}$ regret.

The algorithm is called \SpectralEliminator\/
(Algorithm~\ref{alg:TransductiveElimination})
and works in phases, eliminating the arms
that are not promising. The  phases are defined by the time indexes $t_1=1\leq t_2\leq \dots$
and depend on some parameter $\beta$.
The algorithm is in a spirit similar to the Improved UCB by~\citet{auer2010ucb}.
The main idea
of \SpectralEliminator\/ is to divide the time steps into sets in order
to introduce independence and allow
the Azuma-Hoeffding inequality \cite{azuma1967weighted} to be applied.
In the following theorem we characterize the performance
of \SpectralEliminator\/ and show
that the upper bound on regret has $\sqrt{d}$ improvement 
over \SpectralUCB\/.

\vfill\break

\begin{theorem}
\label{thm:eliminator}
Choose the phases starts as  $t_j=2^{j-1}$. Assume all rewards are in $[0,1]$
and $\|\balpha^*\|_\bLambda\leq C$. For any $\delta>0$, with probability at least
$1-\delta$, the cumulative regret of \SpectralEliminator\/ algorithm run with
parameter $\beta=2R \sqrt{\!14\log(2K\!(1+\log_2 T)/\delta)} + C$ is bounded as:
\begin{align*}
R_T \leq 2 &+ 16\left(2R \sqrt{14 \log\frac{2K(1+\log_2
T)}{\delta}}+C + \frac{1}{2}\right) \\
&\times\sqrt{dT\log_2(T) \log\left(1 + T/\lambda\right)}
\end{align*}
\end{theorem}

\vspace{-0.9em}
\subsection{Scalability and computational complexity}
\label{ssec:scalability}

There are three main computational issues to address in
order to make \SpectralUCB\/ scalable:
the computation of $N$ UCBs, matrix inversion, and obtaining the eigenbasis
which serves as an input to the algorithm.
First, to speed up the computation of $N$ UCBs in each time step, we use lazy
updates technique~\cite{desautels12parallelizing} which maintains a sorted
queue of UCBs and in practice leads to substantial
speed gains. Second, to speed up matrix inversion we do
iterative matrix inversion \cite{zhang2005schur}.

Finally, while the eigendecomposition of a general matrix is 
computationally difficult, Laplacians are symmetric diagonally
dominant (SDD). This enables us to use fast SDD solvers as CMG by
\citet{koutis2011combinatorial}. Furthermore, using CMG we can find good
approximations to the first
$L$ eigenvectors in $\cO(L m \log m)$ time, where $m$ is the number of edges in
the graph (e.g. $m=10N$ in the Flixster experiment).
CMG can easily work with $N$ in millions.
In general, we have $L = N$ but from our experience, a smooth reward function
can
be often approximated by dozens of eigenvectors. In fact, $L$ can be considered
as an upper bound on the number of eigenvectors we actually need.
Furthermore, by choosing small $L$ we not only reduce the complexity of
eigendecomposition but also the complexity of the least-square problem being
solved in each iteration.

Choosing a small $L$ can significantly reduce the computation
but it is important to choose $L$ large enough so that still less than
$J$ eigenvectors are enough. This way, the
problem that we solve is still relevant and our analysis applies.
In short, the problem cannot be solved trivially by choosing first $k$
relevant eigenvectors because $k$ is unknown. Therefore, in practice we choose
the largest $L$ such that our method is able to
run. In Section~\ref{ssec:reduced}, we  demonstrate
that we can obtain good results with relatively small~$L$.
\vspace{-0.4em}
\section{Analysis}
\label{sec:analysis}

The analysis of \SpectralUCB\/  (Section~\ref{sec:regretspectralucb}) has
two main ingredients. The first one
is the derivation of the confidence ellipsoid
for $\hat\balpha$, which is a straightforward update of 
the analysis of OFUL \cite{abbasi2011improved} using
self-normalized martingale inequality (Section~\ref{sec:conf}).
The second part is crucial
to prove that the final regret bound
scales only with the effective dimension $d$
and not with the ambient dimension $D$.
We achieve this by considering
the geometrical properties
of the determinant which hold in our setting
(Section~\ref{sec:effd}).
We also used this result
to upperbound the regret of
\SpectralEliminator\/ (Section~\ref{sec:regret_eliminator}).
  The proofs of the lemmas are in the appendix.

\vspace{-0.4em}
\subsection{Confidence ellipsoid}
\label{sec:conf}

The first two lemmas are by \citet{abbasi2011improved}
and we restate them for the convenience.

\begin{lemma}
\label{lem:selfnorm}
 Let $\bV_t = \bLambda + \sum_{s=1}^{t-1}\bx_{s} \bx_{s}\transpose $
and define $\bxi_t = \sum_{s=1}^{t-1} \varepsilon_s \bx_s$.
With probability at least $1-\delta$, $\forall t\geq 1$:
\begin{align*}
\|\bxi_{t}\|^2_{\bV_t^{-1}}
		\leq&  2 R^2 \log \left(\frac{|\bV_t|^{1/2}}
  {\delta\textbf{}|\bLambda|^{1/2}}\right)
\end{align*}
\end{lemma}
\begin{lemma}\label{lemma:sumx2}		
For any $t$, let $\bV_t = \bLambda + \sum_{s=1}^{t-1}\bx_{s} \bx_{s}\transpose$. Then:
$$\sum_{s = 1}^t\min\left(1,\,\|\bx_s\|^2_{\bV_{s}^{-1}}\right)\leq 2\log\frac{|\bV_{t+1}|}{|\bLambda|}$$
\end{lemma}
%
\vspace{-.4em}

The next lemma is a generalization of Theorem~2 by~\citet{abbasi2011improved}
to the regularization with~$\bLambda$.
The result of this lemma is also used 
for the confidence coefficient $c$
in Algorithm~\ref{alg:TUCB}, which
we upperbound in Section~\ref{sec:effd}
to avoid the computation of determinants.

\begin{lemma}
\label{lem:confinterval}
 Let $\bV_t = \bLambda + \sum_{s=1}^{t-1}\bx_{s} \bx_{s}\transpose$
and $\| \balpha^* \|_{\bLambda} \leq C$. With probability at least $1-\delta$,
for any $\bx$  and $t\geq 1$:
\begin{align*}
|\bx\transpose (\hat \balpha_t-\balpha^*)|\leq
\|\bx  \|_{\bV_t^{-1}} \left(R
\sqrt{2\log \left(\frac{|\bV_t|^{1/2}}
  {\delta\textbf{}|\bLambda|^{1/2}}\right)} + C \right)
\end{align*}
\end{lemma}

\vspace{-.5em}



\vspace{-1.2em}
\subsection{Effective dimension}
\label{sec:effd}

In Section~\ref{sec:conf} we show that several quantities scale with
 $\log(|\bV_t|/|\bLambda|)$, which
can be of the order of $D$. 
Therefore, in this part we present the key ingredient
of our analysis based on the geometrical
properties of determinants (Lemmas~\ref{lemma:2} and~\ref{lemma:detgeo}) to
upperbound $\log(|\bV_t|/|\bLambda|)$ by a term
that scales with $d$ (Lemma~\ref{lem:logdetratio}). Not only this will allow us
to show that the regret
scales with $d$, but it also helps to avoid the costly
computation of the determinants in Algorithm~\ref{alg:TUCB}.

\begin{lemma}\label{lemma:2}
Let $\bLambda=\mbox{diag}(\lambda_1,\dots,\lambda_N)$
be any diagonal matrix with strictly positive entries and for any vectors
$(\bx_s)_{1\leq s< t}$
 such that $\|\bx_s \|_2\leq 1$ for all $1\leq s< t$, we have that the
 determinant $|\bV_{t}|$ of $\bV_{t}=\bLambda + \sum_{s=1}^{t-1} \bx_s \bx_s\transpose$
 is maximized when all $\bx_s$ are  aligned with the axes.
\end{lemma}

\begin{lemma}\label{lemma:detgeo}
Let $\bV_{t}=\bLambda + \sum_{s=1}^{t-1} \bx_s \bx_s\transpose $. Then:
\vspace{-.3em}
$$\log\frac{|\bV_{t}|}{| \bLambda  |} \leq \max \sum_{i=1}^N
\vspace{-.3em}
\log\Big(1+\frac{t_i}{\lambda_i}\Big),$$
where the maximum is taken over all possible positive real
numbers $\{t_1,\dots,t_N\}$, such that $\sum_{i=1}^N t_i = t-1$.
\end{lemma}

\begin{lemma}\label{lem:logdetratio}
Let $d$ be the effective dimension and T be the time horizon of the algorithm. Then:
$$\log\frac{|\bV_{T+1}|}{| \bLambda  |} \leq 2 d \log\left(1+\frac{T}{\lambda}\right)$$
\end{lemma}
\vspace{-1em}
\subsection{Cumulative regret of \SpectralUCB\/}
\label{sec:regretspectralucb}

\begin{proof}[Proof of Theorem~\ref{thm:tucb}]
Let $\bx_{*} = \argmax_{\bx_v} \bx_v\transpose\balpha^*$ and let $\bar r_t$ denote the instantaneous regret at time $t$. With probability at least $1-\delta$, for all $t$:
\begin{align}
\bar r_t  &= \bx_{*}\transpose\balpha^* - \bx_t\transpose\balpha^*  \nonumber \\
	&\leq  \bx_t\transpose\hat\balpha_t + c\|\bx_t\|_{\bV_t^{-1}} - \bx_t\transpose\balpha^* \label{ThRuseOFU} \\
        &\leq  \bx_t\transpose\hat\balpha_t + c\|\bx_t\|_{\bV_t^{-1}} - \bx_t\transpose\hat\balpha_t + c\|\bx_t\|_{\bV_t^{-1}}   \label{ThRuseConf} \\
	& = 2 c \|\bx_t\|_{\bV_t^{-1}}. \nonumber
\end{align}
The inequality \eqref{ThRuseOFU} is by the algorithm design and
reflects the optimistic principle of \SpectralUCB\/. Specifically,
$ \bx_{*}\transpose\hat{\balpha}_t + c\|\bx_{*}\|_{\bV_t^{-1}} \leq \bx_t\transpose\hat\balpha_t + c\|\bx_{t}\|_{\bV_t^{-1}},
$
from which:
$$
 \bx_{*}\transpose\balpha^*  \leq \bx_{*}\transpose\hat{\balpha}_t + c\|\bx_{*}\|_{\bV_t^{-1}}\leq \bx_t\transpose\hat\balpha_t + c\|\bx_{t}\|_{\bV_t^{-1}}  \\
$$

In~\eqref{ThRuseConf} we applied Lemma~\ref{lem:confinterval}:
$
 \bx_t\transpose\hat\balpha_t \leq \bx_{t}\transpose\balpha^* + c\|\bx_t\|_{\bV_t^{-1}}.
$
Finally, by Lemmas~\ref{lemma:sumx2} and~\ref{lem:logdetratio}:
\begin{align*}
R_T &= \sum_{t=1}^T \bar r_t \leq \sum_{t=1}^T \min\left(2,\,2c\|\bx_t\|_{\bV_t^{-1}}\right) \\
&\leq (2+2c)\sum_{t=1}^T \min\left(1,\,\|\bx_t\|_{\bV_t^{-1}}\right)\\
&\leq(2+2c)\sqrt{T\sum_{t=1}^T \min\left(1,\,\|\bx_t\|_{\bV_t^{-1}}^2\right)}\\
&\leq(2+2c)\sqrt{2T\log(|\bV_{T+1}|/|\bLambda|)} \\
&\leq(2+2c)\sqrt{4dT\log\left(1 + T/\lambda\right)}\\
\end{align*}

\vspace{-2em}

By plugging $c$, we get that with probability at least $1-\delta$, the theorem holds.
\end{proof}

\begin{remark}
\label{rem:linucb}
Notice that if we set $\bLambda =
\bI$ in Algorithm~\ref{alg:TUCB}, we recover LinUCB. Since
$\log(|\bV_{T+1}|/|\bLambda|)$ can be upperbounded by $D \log T$
\cite{abbasi2011improved}, we obtain $\tilde\cO(D\sqrt{T})$ upper bound
of regret of LinUCB as a corollary of Theorem~\ref{thm:tucb}.
\end{remark}

\vspace{-0.4em}
\subsection{Cumulative regret of \SpectralEliminator\/}
\label{sec:regret_eliminator}

The probability space induced by the rewards $r_1,r_2,\dots$ can be decomposed
as a product of independent probability spaces induced by rewards in each phase
$[t_j, t_{j+1}-1]$. Denote by $\cF_j$ the $\sigma$-algebra generated by
the rewards $r_1,\dots,r_{t_{j+1}-1}$, i.e.,~received before and during the
phase~$j$.
We have the following two lemmas for any phase $j$, where $\overline{\bV}_j = \bLambda + \sum_{t = t_{j-1}}^{t_j-1}\bx_t\bx_t\transpose$. 

\begin{lemma}\label{lem:3}
For any fixed $\bx\in\R^N$, any $\delta>0$, and
 $\beta(\delta)=2R \sqrt{14\log(2/\delta)} + \|\balpha^*\|_\bLambda  $, we have 
for any phase~$j$:
$$ \P\Big( |\bx\transpose (\hat\balpha_j-\balpha^*)| \leq \|\bx\|_{\overline{\bV}_{j}^{-1}}
\beta(\delta) \Big) \geq 1-\delta$$
\end{lemma}

\begin{lemma}\label{lem:4}For all $\bx\in A_j$, $j>1$, we have:
$$\min\left(1,\,\|\bx\|_{\overline{\bV}_{j}^{-1}}\right) \leq \frac{1}{t_{j}-t_{j-1}} \sum_{s=t_{j-1}}^{t_{j}-1}
\min\left(1,\,\|\bx_s\|_{\bV_{s}^{-1}}\right)$$
\end{lemma}

\vspace{-.5em}

\begin{figure*}
 \begin{center}
\vskip -0.5em
\includegraphics[width=0.66\columnwidth]
{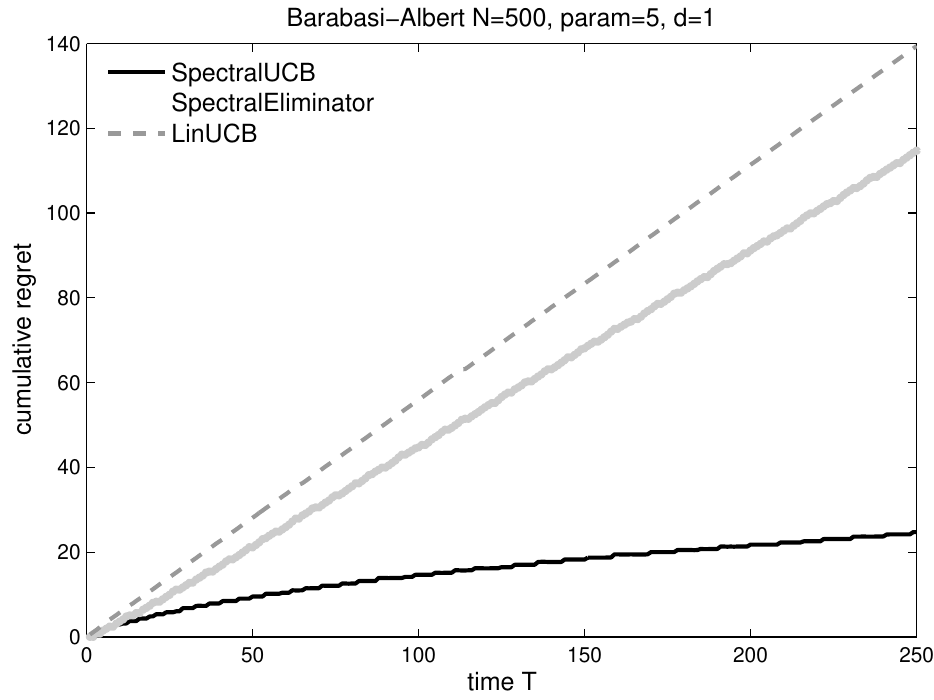}
\includegraphics[width=0.66\columnwidth]
{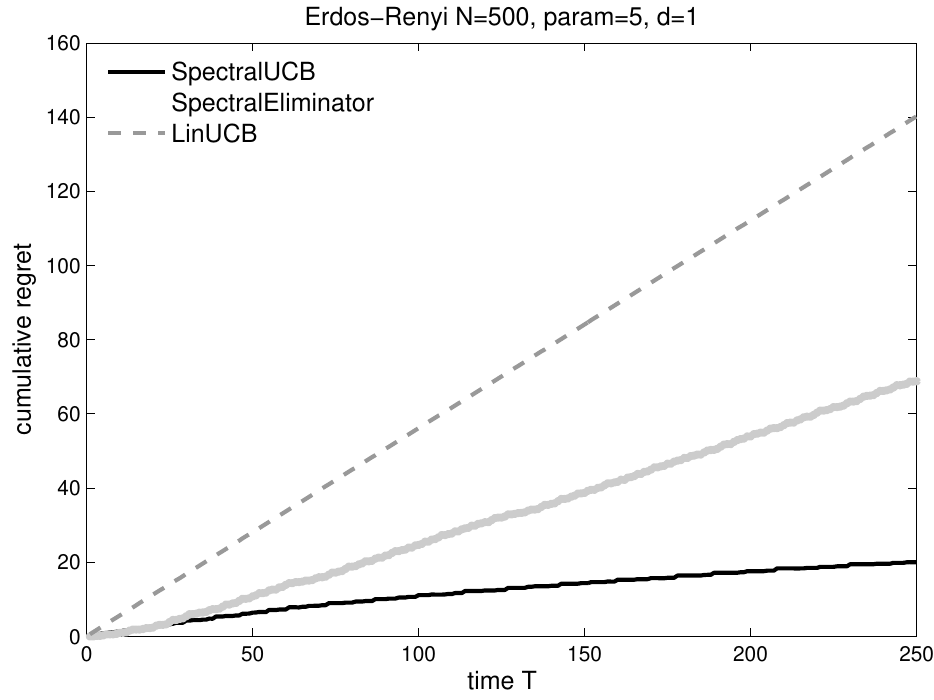}
\includegraphics[width=0.66\columnwidth]
{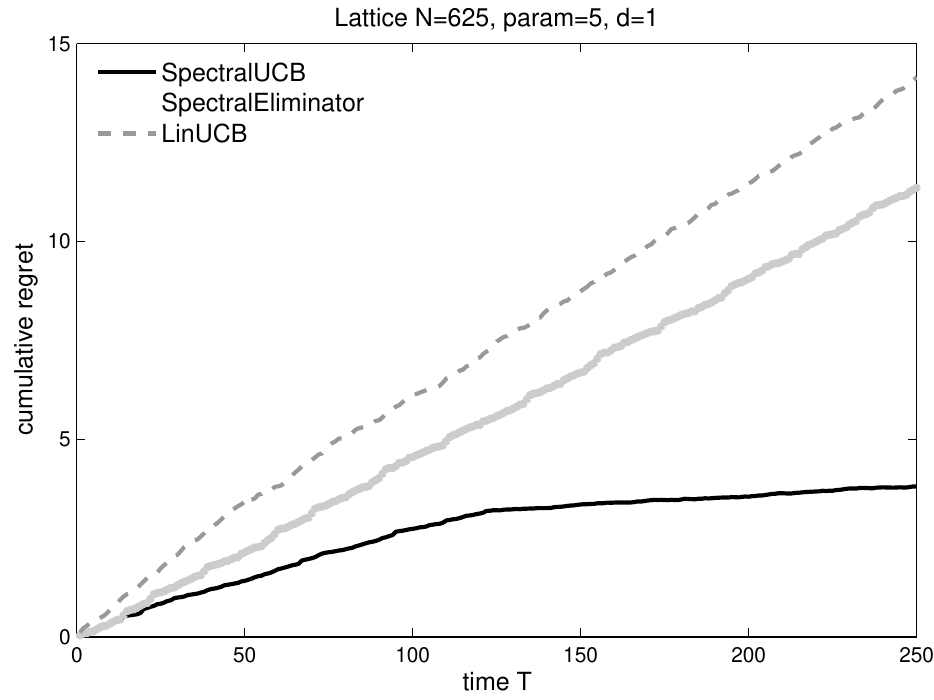}
\vskip -0.5em
\caption{Cumulative regret for random graphs models}
\label{fig:baall}
 \end{center}
 \vskip -1.5em
 \end{figure*}
\vspace{-0.4em}

Now we are ready to upperbound the cumulative regret
of \SpectralEliminator\/.
 \vspace{-1em}
\begin{proof}[Proof of Theorem~\ref{thm:eliminator}]
Let $J = \lfloor \log_2 T\rfloor + 1$ and $t_j = 2^{j-1}$. We have:
\begin{align*}
R_T &= \sum_{t=1}^T  \langle \bx_*-\bx_t, \balpha^* \rangle 		\\
&\leq 2 +  \sum_{j=2}^{J } \sum_{t=t_j}^{t_{j+1}-1} \min(2,\,\langle \bx_*-\bx_t, 
\balpha^* \rangle) \\
&\leq 2+\sum_{j=2}^{J } \sum_{t=t_j}^{t_{j+1}-1}\min\Big(2,\,\langle \bx_*-\bx_t, \hat\balpha_j \rangle \\
&\quad\qquad\qquad+\left(\|\bx_*\|_{\overline{\bV}_{j}^{-1}}+\|\bx_t\|_{\overline{\bV}_{j}^{-1}}\right)\beta(\delta')\Big),
\end{align*}
in an event $\Omega$ of probability $1-\delta$, where we used Lemma~\ref{lem:3} in the last inequality for $\delta' = \delta/(KJ)$. By definition of the action subset $A_j$ at phase $j>1$, under $\Omega$, we have:
$$\langle \bx_*-\bx_t, \hat\balpha_j \rangle \leq \left(\|\bx_*\|_{\overline{\bV}_j^{-1}}+\|\bx_t\|_{\overline{\bV}_j^{-1}}\right)\beta(\delta'),$$ since $\bx_*\in A_j$ for all $j\leq J$.
By previous two lemmas and Cauchy-Schwarz inequality:

\vspace{-1.8em}

\begin{align*}
&R_T\leq  2 + \sum_{j = 2}^J\sum_{t = t_j}^{t_{j+1}-1}\min\left(2,\, 4 \beta(\delta')\|\bx_t\|_{\overline{\bV}_{j}^{-1}}\right)		\\
&\leq 2 + (4\beta(\delta')+2)\sum_{j = 2}^J \sum_{t = t_j}^{t_{j+1}-1}\min\left(1,\,\|\bx_t\|_{\bV_t^{-1}}\right)		\\
&\leq 2 + (4\beta(\delta')+2)\sum_{j = 2}^J \frac{t_{j+1}-t_j}{t_j-t_{j-1}}\sum_{t = t_{j-1}}^{t_{j}-1}\min\left(1,\,\|\bx_t\|_{\bV_t^{-1}}\right)		\\
&\leq 2 + (8\beta(\delta')+4)\sum_{j = 2}^J \sum_{t = t_{j-1}}^{t_{j}-1}\min\left(1,\,\|\bx_t\|_{\bV_t^{-1}}\right)		\\
&\leq 2 + (8\beta(\delta')+4)\sqrt{T\sum_{j = 2}^J \sum_{t = t_{j-1}}^{t_{j}-1}\min\left(1,\,\|\bx_t\|_{\bV_t^{-1}}^2\right)}\\
&\leq 2 + (8\beta(\delta')+4)\sqrt{T\sum_{j = 2}^J 2\log\frac{|\overline{\bV_j}|}{|\bLambda|}}\\
&\leq 2 + (16\beta(\delta')+8)\sqrt{dT\log_2(T) \log\left(1 + \frac{T}{\lambda}\right)}
\end{align*}

Finally, using $J = 1 + \lfloor \log_2T\rfloor$, $\delta'=\delta/(KJ)$, and $\beta(\delta')\leq\beta(\delta/(K(1+\log_2T)))$, we obtain the result of Theorem~\ref{thm:eliminator}.
\end{proof}


\begin{remark}
If we set $\bLambda = \bI$ in~Algorithm~\ref{alg:TransductiveElimination} 
 as in Remark~\ref{rem:linucb}, we get a new algorithm,
\LinearEliminator\/,
which is a competitor to SupLinRel~\cite{auer2002using}
and as a corollary to Theorem~\ref{thm:eliminator}
also enjoys $\tilde\cO(\sqrt{DT})$ upper bound on the cumulative regret.
On the other hand, compared to SupLinRel, \LinearEliminator\/ and its analysis
are
significantly simpler and elegant.
\end{remark}

\section{Experiments}\label{sec:exp}

We evaluated our algorithms
and compared them to LinUCB.
In all experiments we set $\delta$ to 0.001 and $R$ to 0.01.
For \SpectralUCB\/ and \SpectralEliminator\/ we set $\bLambda$ to
$\bLambda_{\cL} + \lambda \bI$
with $\lambda = 0.01$. For LinUCB we
regularized with $\lambda \bI$ with $\lambda = 1$.
Our results are robust to small perturbations of all learning parameters.
We also performed experiments with
SupLinRel, SupLinUCB, SupSpectralUCB\footnote{an equivalent of SupLinUCB with spectral regularization},
but due to the known reasons \cite{chu2011contextual} these
algorithms are not efficient\footnote{at least for the sizes of $N$ and $T$ that
we deal with} and they were always outperformed by \SpectralUCB\/ and LinUCB.

\subsection{Random graph models}
To simulate realistic graph structures, we generated graphs of $N$ nodes
using three models that are commonly used in the social networks modeling.
First, we considered the widely known
\textit{Erd\H os-R\' enyi} (ER) model. We sampled the edges in the ER model
independently with probability 3\%. Second, we considered the
\textit{Barab\'asi-Albert} (BA)
model \citeyearpar{barabasi1999emergence}, with
the degree parameter 3.
BA models are commonly used for modeling real networks
due to their \textit{preferential attachment} property.
Finally, we considered graphs where the edge structure forms a
regular \textit{lattice}.

For all the graph models we assigned uniformly random weights to their edges.
Then, we randomly generated $k$-sparse vector $\balpha^*$ of $N$ weights,
$k \ll N$ and defined the true graph function as $f = \bQ\balpha^*$,
where $\bQ$ is the matrix
of eigenvectors from the eigendecomposition of the graph Laplacian.
We ran the algorithms in the desired $T < N$ regime,
with $N=500$ ($N=5^4$ for the lattice), $T = 250$, and $k=5$.
 Figure~\ref{fig:baall} 
shows that the regret of LinUCB for all three models has within first $T$ steps
still a linear trend unlike \SpectralUCB\/ that performs much better.

Unfortunately, even though the regret bound Spectral Eliminator
is asymptotically better, it was outperformed by \SpectralUCB\/.
This is similar to the linear case when LinUCB
outperforms\footnote{For example, these algorithms do not use all the rewards
obtained for the estimation of $\hat\balpha$.}
 SupLinUCB in practice
\cite{chu2011contextual} while it is an open problem
whether LinUCB can be shown to have a better regret.

\begin{figure*}
 \begin{center}
\includegraphics[viewport = 0 301 608 489,clip, width=1\columnwidth]
{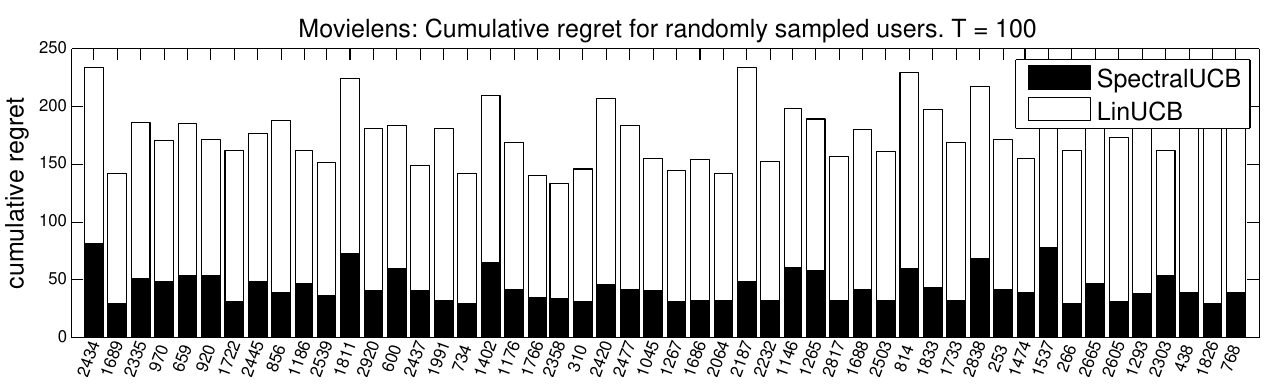}
\includegraphics[viewport = 0 301 608 489,clip,
width=1\columnwidth]{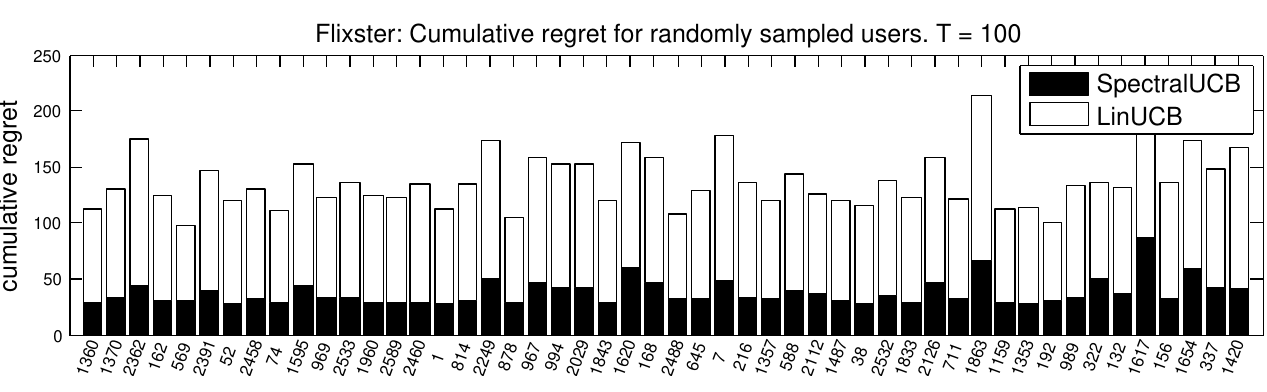}
\caption{MovieLens and Flixster: Cumulative regret for 50 randomly chosen
users.
Horizontal axis shows the user number.}
\label{fig:flixster_many}
 \end{center}
 \end{figure*}

\subsection{MovieLens experiments}
\label{ssec:movielens}

In this experiment we took user preferences and the similarity graph over
movies from the MovieLens dataset \cite{movielens}, a dataset of
6k users who rated one million movies. We divide the dataset into two
equally-sized parts. The first dataset is used to build our model of users, the
rating that user $i$ assigns to movie $j$. We factor the user-item matrix using
low-rank matrix factorization \cite{keshavan2009matrix} as $\bM \approx \bU
\bV'$,
a
standard
approach to collaborative filtering. The rating that the user $i$ assigns to
movie
$j$ is estimated as $\hat{r}_{i, j} = \langle\mathbf{u}_i, \mathbf{v}_j\rangle$,
where $\mathbf{u}_i$ is the $i$-th row of $\bU$ and $\mathbf{v}_j$ is the $j$-th
row of $\bV$. The rating $\hat{r}_{i, j}$ is the payoff of pulling arm $j$ when
recommending to user $i$. The second dataset is used to build our similarity
graph over movies. We factor the dataset in the same way as the first dataset.
The graph contains an edge between movies $i$ and $i'$ if the movie $i'$ is
among 10 nearest neighbors of the
movie $i$ in the latent space of items $\bV$. The weight on all edges is one.
Notice that if two items are close in the item space, then their expected
rating is expected to be similar. However, the opposite is not true. If two
items have a similar expected rating, they do not have to be close in the item
space. In other words, we take advantage of ratings  but do
not hardwire the two similarly rated items to be similar.

In Figure~\ref{fig:flixster_many}, we sampled 50 users and evaluated the regret
of both algorithms for $T=100$. Here \SpectralUCB\/ suffers only about one
fourth of regret over LinUCB, specifically 43.4611 vs.\,133.0996 on average.

\subsection{Flixster Experiments}
\label{sec:movielens}

We also performed experiments on users preferences from the movie
recommendation
website Flixster. The social network of the users was crawled by
\citet{jamali2010matrix} and then clustered by \citet{graclus} to obtain a
strongly connected subgraph. We extracted a subset of users and movies, where
each movie has at least 30 ratings and each user rated at least 30 movies. This
resulted in a dataset of 4546 movies and 5202 users.
As with MovieLens dataset we completed the missing
ratings by a low-rank matrix factorization and used it construct a 10-NN
similarity graph.

Again in Figure~\ref{fig:flixster_many}, we
sampled 50 users and evaluated the regret of both algorithms for $T=100$.
On average, \SpectralUCB\/ suffers only about one third
of regret over LinUCB, specifically 37.6499 vs.\,~99.8730 on average.

\subsection{Reduced basis}
\label{ssec:reduced}

As discussed in Section~\ref{ssec:scalability}, one
can decrease the computational complexity
and thus increase the scalability by only
extracting first $L\ll N$ eigenvectors
of the graph Laplacian. First, the computational
complexity of such operation is $\cO(L m \log m)$,
where $m$ is the number of edges. Second, the least-squares
problem that we have to do in each time step of the
algorithm is only $L$ dimensional.

In Figure~\ref{fig:reduced} we plot
the cumulative regret and the total
computational time in seconds (log scale) for a single user from the MovieLens
dataset. We varied $L$ as 20, 200, and 2000
 which corresponds to about $1\%$,
$10\%$ and $100\%$ of basis functions ($N=2019$).
The total computational time also includes
the computational savings from
lazy updates and iterative matrix inversion.
We see that with $10\%$ of the eigenvectors we can achieve
similar performance as for the full set for the fraction of the computational
time.

\begin{figure}
 \begin{center}
\includegraphics[width=0.49\columnwidth]
{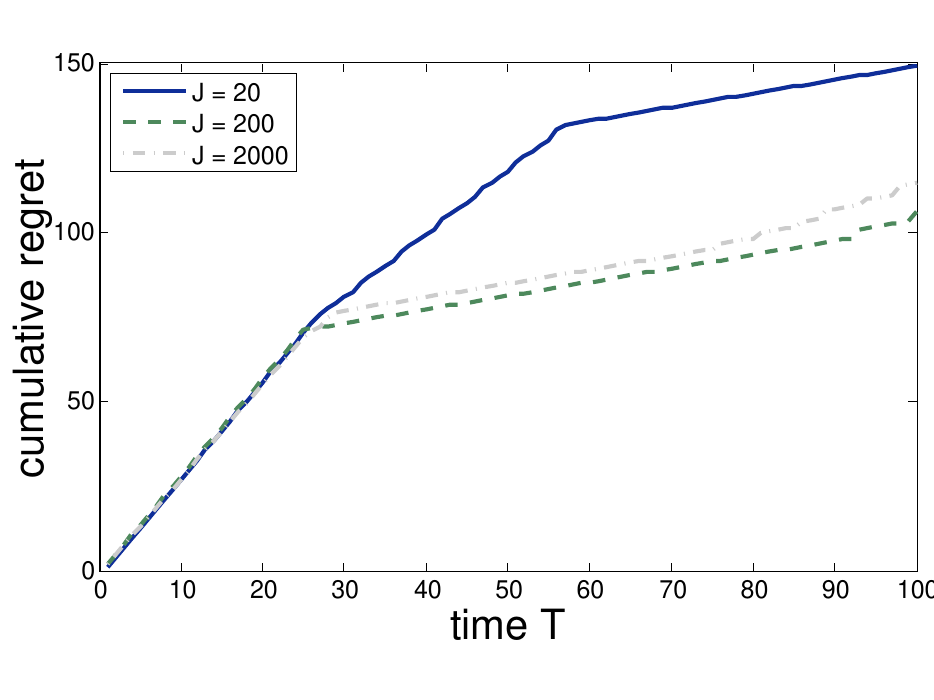}
\includegraphics[width=0.49\columnwidth]
{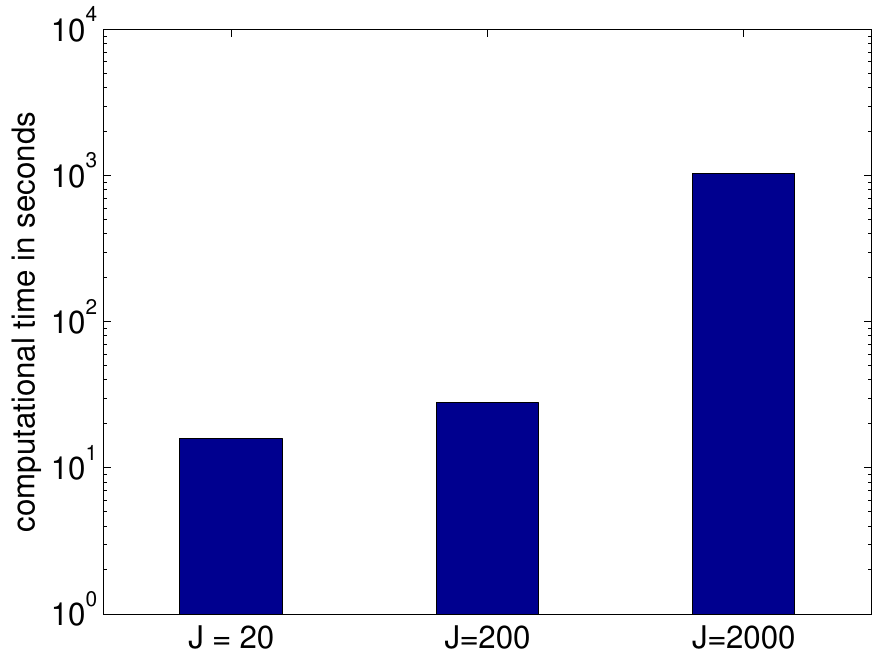}
\caption{Regret and computational time with reduced basis}
\label{fig:reduced}
 \end{center}
 \end{figure}

  \vspace{0.4em}
\section{Conclusion}
\label{sec:concl}
We presented spectral bandit setting
inspired mostly by the applications in
the recommender systems and targeted advertisement in social networks.
In this setting, we are asked to repeatedly
maximize an unknown graph function,
 assumed to be smooth
on a given similarity graph.
Traditional linear bandit
algorithm can be applied but their regret
scales with the ambient
dimension $D$, either linearly or as a square root,
which can be very large.

Therefore, we introduced
two algorithms, \SpectralUCB\/ and \SpectralEliminator\/, for which
the regret only scales with 
effective dimension $d$
which is typically much smaller than $D$
for real-world graphs.
We demonstrated that \SpectralUCB\/ delivers
desired benefit  for the graphs generated by
Barab\'asi--Albert, Erd\H os-R\' enyi,
and regular lattice models; and for the movie
rating data from the MovieLens and Flixster social networks.
In the future, we plan to extend this
work to a sparse setting when the smooth
function is assumed to be a linear combination
of only finite number of eigenvectors.

  \vspace{0.4em}
\section{Acknowledgements}
\label{sec:Acknowledgements}
We would like to thank Yiannis Koutis for his great help with the efficient
computation of eigenvectors. We thank Andreas Krause for suggesting
the lazy updates of UCBs. We would also
like to thank Giovanni Zappella, Claudio Gentile, and especially Alessandro
Lazaric for helpful discussions.
The research presented in this paper was supported by French Ministry of
Higher Education and Research, by European Community's
Seventh Framework Programme (FP7/2007-2013) under grant agreement n$^{\rm
o}$270327 (project CompLACS), and by Intel Corporation.

\urlstyle{same}
 \setlength{\bibsep}{4.95pt}
\clearpage
 \bibliography{../../../library}

\begin{thebibliography}{33}
\providecommand{\natexlab}[1]{#1}
\providecommand{\url}[1]{\texttt{#1}}
\expandafter\ifx\csname urlstyle\endcsname\relax
  \providecommand{\doi}[1]{doi: #1}\else
  \providecommand{\doi}{doi: \begingroup \urlstyle{rm}\Url}\fi

\bibitem[Abbasi-Yadkori et~al.(2011)Abbasi-Yadkori, P\'{a}l, and
  Szepesv\'{a}ri]{abbasi2011improved}
Abbasi-Yadkori, Y, P\'{a}l, D, and Szepesv\'{a}ri, C.
\newblock {Improved Algorithms for Linear Stochastic Bandits}.
\newblock In \emph{Neural Information Processing Systems (NeurIPS)}. 2011.

\bibitem[Abernethy et~al.(2008)Abernethy, Hazan, and
  Rakhlin]{abernethy2008competing}
Abernethy, J.~D, Hazan, E, and Rakhlin, A.
\newblock {Competing in the Dark: An Efficient Algorithm for Bandit Linear
  Optimization.}
\newblock In \emph{Conference on Learning Theory (COLT)}, 2008.

\bibitem[Alon et~al.(2013)Alon, Cesa-Bianchi, Gentile, and
  Mansour]{alon2013from}
Alon, N, Cesa-Bianchi, N, Gentile, C, and Mansour, Y.
\newblock {From Bandits to Experts: A Tale of Domination and Independence}.
\newblock In \emph{Neural Information Processing Systems (NeurIPS)}, 2013.

\bibitem[Auer(2002)]{auer2002using}
Auer, P.
\newblock {Using confidence bounds for exploitation-exploration trade-offs}.
\newblock \emph{Journal of Machine Learning Research}, 3:\penalty0 397--422,
  March 2002.
\newblock ISSN 1532-4435.

\bibitem[Auer \& Ortner(2010)Auer and Ortner]{auer2010ucb}
Auer, P and Ortner, R.
\newblock {UCB Revisited: Improved Regret Bounds for the Stochastic Multi-Armed
  Bandit Problem}.
\newblock \emph{Periodica Mathematica Hungarica}, 2010.

\bibitem[Azuma(1967)]{azuma1967weighted}
Azuma, K.
\newblock {Weighted sums of certain dependent random variables}.
\newblock \emph{Tohoku Mathematical Journal}, 19\penalty0 (3):\penalty0
  357--367, 1967.

\bibitem[Barab\'{a}si \& Albert(1999)Barab\'{a}si and
  Albert]{barabasi1999emergence}
Barab\'{a}si, A.-L and Albert, R.
\newblock {Emergence of scaling in random networks}.
\newblock \emph{Science}, 286:\penalty0 11, 1999.

\bibitem[Belkin et~al.(2004)Belkin, Matveeva, and
  Niyogi]{belkin2004regularization}
Belkin, M, Matveeva, I, and Niyogi, P.
\newblock {Regularization and Semi-Supervised Learning on Large Graphs}.
\newblock In \emph{Conference on Learning Theory (COLT)}, 2004.

\bibitem[Belkin et~al.(2006)Belkin, Niyogi, and Sindhwani]{belkin2006manifold}
Belkin, M, Niyogi, P, and Sindhwani, V.
\newblock {Manifold Regularization: A Geometric Framework for Learning from
  Labeled and Unlabeled Examples}.
\newblock \emph{Journal of Machine Learning Research}, 7:\penalty0 2399--2434,
  2006.

\bibitem[Billsus et~al.(2000)Billsus, Pazzani, and Chen]{billsus2000learning}
Billsus, D, Pazzani, M.~J, and Chen, J.
\newblock {A learning agent for wireless news access}.
\newblock In \emph{IUI}, pp.\  33--36, 2000.

\bibitem[Bubeck et~al.(2011)Bubeck, Munos, Stoltz, and Szepesvari]{bubeck2011x}
Bubeck, S, Munos, R, Stoltz, G, and Szepesvari, C.
\newblock {X-armed bandits}.
\newblock \emph{Journal of Machine Learning Research}, 12:\penalty0 1587--1627,
  2011.

\bibitem[Bubeck et~al.(2012)Bubeck, Cesa-Bianchi, and
  Kakade]{bubeck2012towards}
Bubeck, S, Cesa-Bianchi, N, and Kakade, S.
\newblock {Towards minimax policies for online linear optimization with bandit
  feedback}.
\newblock In \emph{Conference on Learning Theory (COLT)}, 2012.

\bibitem[Caron et~al.(2012)Caron, Kveton, Lelarge, and
  Bhagat]{caron2012leveraging}
Caron, S, Kveton, B, Lelarge, M, and Bhagat, S.
\newblock {Leveraging Side Observations in Stochastic Bandits.}
\newblock In \emph{Conference on Uncertainty in Artificial Intelligence (UAI)}, pp.\  142--151,
  2012.

\bibitem[Cesa-Bianchi et~al.(2013)Cesa-Bianchi, Gentile, and
  Zappella]{cesa-bianchi2013gang}
Cesa-Bianchi, N, Gentile, C, and Zappella, G.
\newblock {A Gang of Bandits}.
\newblock In \emph{Neural Information Processing Systems (NeurIPS)}, 2013.

\bibitem[Chau et~al.(2011)Chau, Kittur, Hong, and Faloutsos]{chau2011apolo}
Chau, D.~H, Kittur, A, Hong, J.~I, and Faloutsos, C.
\newblock {Apolo: making sense of large network data by combining rich user
  interaction and machine learning}.
\newblock In \emph{Conference on Human Factors in Computing Systems (CHI)}, 2011.

\bibitem[Chu et~al.(2011)Chu, Li, Reyzin, and Schapire]{chu2011contextual}
Chu, L, Li, L, Reyzin, L, and Schapire, R.
\newblock {Contextual Bandits with Linear Payoff Functions}.
\newblock In \emph{International Conference on Artificial Intelligence and Statistics (AISTATS)}, 2011.

\bibitem[Dani et~al.(2008)Dani, Hayes, and Kakade]{dani2008stochastic}
Dani, V, Hayes, T.~P, and Kakade, S.~M.
\newblock {Stochastic Linear Optimization under Bandit Feedback}.
\newblock In \emph{Conference on Learning Theory (COLT)}, 2008.

\bibitem[Desautels et~al.(2012)Desautels, Krause, and
  Burdick]{desautels12parallelizing}
Desautels, T, Krause, A, and Burdick, J.
\newblock {Parallelizing Exploration-Exploitation Tradeoffs with Gaussian
  Process Bandit Optimization}.
\newblock In \emph{International Conference on Machine Learning (ICML)}, 2012.

\bibitem[Graclus(2013)]{graclus}
Graclus.
\newblock {Graclus}, 2013.
\newblock URL \url{www.cs.utexas.edu/users/dml/Software/graclus.html}.

\bibitem[Jamali \& Ester(2010)Jamali and Ester]{jamali2010matrix}
Jamali, M and Ester, M.
\newblock {A matrix factorization technique with trust propagation for
  recommendation in social networks}.
\newblock In \emph{ACM Conference on Recommender Systems (RecSys)}. ACM, 2010.

\bibitem[Jannach et~al.(2010)Jannach, Zanker, Felfernig, and
  Friedrich]{jannach2010recommender}
Jannach, D, Zanker, M, Felfernig, A, and Friedrich, G.
\newblock \emph{{Recommender Systems: An Introduction}}.
\newblock Cambridge University Press, 2010.

\bibitem[Keshavan et~al.(2009)Keshavan, Oh, and Montanari]{keshavan2009matrix}
Keshavan, R, Oh, S, and Montanari, A.
\newblock {Matrix Completion from a Few Entries}.
\newblock In \emph{IEEE International Symposium on Information Theory}, pp.\
  324--328, 2009.

\bibitem[Kleinberg et~al.(2008)Kleinberg, Slivkins, and
  Upfal]{kleinberg2008multi}
Kleinberg, R, Slivkins, A, and Upfal, E.
\newblock {Multi-armed bandit problems in metric spaces}.
\newblock In \emph{ACM Symposium on Theory of Computing (STOC)}, 2008.

\bibitem[Koutis et~al.(2011)Koutis, Miller, and
  Tolliver]{koutis2011combinatorial}
Koutis, I, Miller, G.~L, and Tolliver, D.
\newblock {Combinatorial preconditioners and multilevel solvers for problems in
  computer vision and image processing}.
\newblock \emph{Computer Vision and Image Understanding}, 115:\penalty0
  1638--1646, 2011.

\bibitem[Lam \& Herlocker(2012)Lam and Herlocker]{movielens}
Lam, S and Herlocker, J.
\newblock {MovieLens 1M Dataset}.
\newblock http://www.grouplens.org/node/12, 2012.

\bibitem[Li et~al.(2010)Li, Chu, Langford, and Schapire]{li2010contextual}
Li, L, Chu, W, Langford, J, and Schapire, R.~E.
\newblock {A Contextual-Bandit Approach to Personalized News Article
  Recommendation}.
\newblock \emph{World Wide Web Conference (WWW)}, 2010.

\bibitem[McPherson et~al.(2001)McPherson, Smith-Lovin, and
  Cook]{mcpherson2001birds}
McPherson, M, Smith-Lovin, L, and Cook, J.
\newblock {Birds of a Feather: Homophily in Social Networks}.
\newblock \emph{Annual Review of Sociology}, 27:\penalty0 415--444, 2001.

\bibitem[Shamir(2011)]{variant2011shamir}
Shamir, O.
\newblock {A Variant of Azuma's Inequality for Martingales with Subgaussian
  Tails}.
\newblock \emph{CoRR}, abs/1110.2, 2011.

\bibitem[Slivkins(2009)]{slivkins2009contextual}
Slivkins, A.
\newblock {Contextual Bandits with Similarity Information}.
\newblock \emph{Conference on Learning Theory (COLT)},
  pp.\  1--27, 2009.

\bibitem[Srinivas et~al.(2010)Srinivas, Krause, Kakade, and
  Seeger]{srinivas2009gaussian}
Srinivas, N, Krause, A, Kakade, S, and Seeger, M.
\newblock {Gaussian Process Optimization in the Bandit Setting: No Regret and
  Experimental Design}.
\newblock \emph{International Conference on Machine Learning (ICML)},
  2010.

\bibitem[Valko et~al.(2013)Valko, Korda, Munos, Flaounas, and
  Cristianini]{valko2013finite}
Valko, M, Korda, N, Munos, R, Flaounas, I, and Cristianini, N.
\newblock {Finite-Time Analysis of Kernelised Contextual Bandits}.
\newblock In \emph{Conference on Uncertainty in Artificial Intelligence (UAI)}, 2013.

\bibitem[Zhang(2005)]{zhang2005schur}
Zhang, F.
\newblock \emph{{The Schur complement and its applications}}, volume~4.
\newblock Springer, 2005.

\bibitem[Zhu(2008)]{zhu2008semi-supervised}
Zhu, X.
\newblock {Semi-Supervised Learning Literature Survey}.
\newblock Technical Report 1530, U. of Wisconsin-Madison, 2008.

\end{thebibliography}
 \bibliographystyle{icml2014}

\clearpage

\nobalance
\appendix

\section*{Supplementary Material}
\textbf{Confidence ellipsoid}

\begin{lemma}\label{lem:updateInverse}
For any symmetric, positive semi-definite matrix $\bX$ and any vector \bu:

\vspace{-1.5em}

$$(\bX + \bu\bu\transpose)^{-1} \prec \bX^{-1}$$
\end{lemma}

\begin{proof}
For any vector $\by$, by Sherman--Morrison formula:
\begin{align*}
-\frac{\left(\bu\transpose\bX^{-1}\by\right)\transpose\left(\bu\transpose\bX^{-1}\by\right)}{1+\bu\transpose\bX^{-1}\bu}&\leq 0	\\
\by\transpose\left(\bX^{-1}-\frac{\bX^{-1}\bu\bu\transpose\bX^{-1}}{1+\bu\transpose\bX^{-1}\bu}\right)\by&\leq \by\transpose\bX^{-1}\by	\\
\by\transpose\left(\bX+\bu\bu\transpose\right)^{-1}\by&\leq \by\transpose\bX^{-1}\by
\end{align*}

\vspace{-2em}

\end{proof}

\setcounter{lemma}{3-1}
\begin{lemma}
 Let $\bV_t =  \bLambda + \sum_{s=1}^{t-1}\bx_s\bx_s\transpose$
and $\| \balpha^* \|_{\bLambda} \leq C$. With probability at least $1-\delta$, 
for any $\bx$  and $t\geq 1$:
\begin{align*}
|\bx\transpose (\hat \balpha_t-\balpha^*)|\leq
\|\bx  \|_{\bV_t^{-1}} \left(R
\sqrt{2\log \left(\frac{|\bV_t|^{1/2}}
  {\delta\textbf{}|\bLambda|^{1/2}}\right)} + C \right)
\end{align*}
\end{lemma}

\begin{proof}[Proof of Lemma~\ref{lem:confinterval}]
We have:
\begin{align*}
|\bx\transpose (\hat \balpha_t-\balpha^*)| &= |\bx\transpose (-\bV_t^{-1}
\bLambda \balpha^* + \bV_t^{-1} \bxi_{t} )| \\
  &\leq |\bx\transpose \bV_t^{-1} \bLambda \balpha^* | + | \bx\transpose
\bV_t^{-1} \bxi_{t} | \\
  &\leq \langle \bx \transpose, \bLambda \balpha^*\rangle_{\bV_t^{-1}}  +
\langle\bx, \bxi_{t}\rangle_{\bV_t^{-1}} \\
  &\leq \|\bx  \|_{\bV_t^{-1}} \left(\|\bxi_{t} \|_{\bV_t^{-1}} + \|\bLambda
\balpha^* \|_{\bV_t^{-1}} \right),
\end{align*}
where we used Cauchy-Schwarz inequality in the last step.
Now we bound $\|\bxi_{t} \|_{\bV_t^{-1}}$ by Lemma~\ref{lem:selfnorm} and notice
that:
\begin{align*}
\|\bLambda \balpha^* \|_{\bV_t^{-1}}
&= \sqrt{(\balpha^*)\transpose \bLambda \bV_t^{-1}  \bLambda \balpha^*} \\
& \leq \sqrt{(\balpha^*)\transpose  \bLambda \balpha^*}
= \| \balpha^* \|_{\bLambda} \leq C
 \end{align*}
\end{proof}



\textbf{Effective dimension}
\setcounter{lemma}{10-1}
\begin{lemma}\label{lemma:1}
For any real positive-definite matrix $A$ with only simple eigenvalue
multiplicities
and any vector $\bx$ such that $\|\bx \|_2\leq 1$ we have that the determinant
$|\bA+\bx \bx\transpose |$ is maximized by a vector $\bx$ which is aligned with
an eigenvector of $\bA$.
\end{lemma}

\begin{proof}[Proof of Lemma~\ref{lemma:1}]
Using Sylvester's determinant theorem, we have:
$$|\bA+\bx \bx\transpose | = |\bA| |\bI+\bA^{-1}\bx\bx\transpose| = |\bA|
(1+\bx\transpose \bA^{-1} \bx)$$
From the spectral theorem, there exists an orthonormal matrix $\bU$, the columns
of which are the eigenvectors of $\bA$; such that $\bA=\bU \bD \bU\transpose$
with
$\bD$  being a diagonal matrix with the positive eigenvalues of $\bA$ on the
diagonal. Thus:
\begin{align*}
\max_{\|\bx\|_2\leq 1} \bx\transpose \bA^{-1} \bx
&= \max_{\|\bx\|_2\leq 1} \bx\transpose \bU \bD^{-1} \bU\transpose \bx\\
&= \max_{\|\by\|_2\leq 1} \by\transpose \bD^{-1} \by,
\end{align*}

since $\bU$ is a bijection from $\{\bx, \|\bx\|_2\leq 1\}$ to itself.

Since there are no multiplicities, it is easy to see that the quadratic mapping
$\by\mapsto \by\transpose \bD^{-1} \by$  is maximized
(under the constraint $\|\by \|_2\leq 1$) by a canonical vector
$\be_I$ corresponding to the lowest diagonal entry $I$ of $\bD$.
Thus the maximum of $\bx\mapsto \bx\transpose \bA^{-1} \bx$ is reached for $\bU
\be_I$, which is the eigenvector of $\bA$ corresponding to its lowest
eigenvalue.
\end{proof}

\setcounter{lemma}{4-1}

\begin{lemma}
Let $\bLambda=\mbox{diag}(\lambda_1,\dots,\lambda_N)$
be any diagonal matrix with strictly positive entries. Then for any vectors
$(\bx_s)_{1\leq s< t}$,
 such that $\|\bx_s \|_2\leq 1$ for all $1\leq s< t$, we have that the
 determinant $|\bV_t|$ of $\bV_t=\bLambda + \sum_{s=1}^{t-1} \bx_s \bx_s\transpose$
 is maximized when all $\bx_s$ are  aligned with the axes.
\end{lemma}

\begin{proof}[Proof of Lemma~\ref{lemma:2}]
Let us write $d(\bx_1,\dots,\bx_{t-1})=|\bV_t|$ the determinant of $\bV_t$. We want
to characterize:
$$\max_{\bx_1,\dots,\bx_{t-1}: \|\bx_s\|_2\leq 1, \forall 1\leq s< t}
d(\bx_1,\dots,\bx_{t-1})$$




For any $1\leq i< t$, let us define:
$$\bV_{-i} = \bLambda + \sum\limits_{{\begin{array}{c} \\[-1.5em]
\scriptstyle{s=1}\\[-0.5em] \scriptstyle{s\neq i}\end{array}}}^{t-1}
\bx_s\bx_s\transpose$$
We have that
$\bV_t = \bV_{-i} + \bx_i\bx_i\transpose$. Consider the case with only simple
eigenvalue
multiplicities.
In this case, Lemma~\ref{lemma:1} implies that $\bx_i \mapsto
d(\bx_1,\dots,\bx_i,\dots,\bx_{t-1})$ is
 maximized when $\bx_i$ is aligned with an eigenvector of $\bV_{-i}$.
Thus all $\bx_s$, for $1\leq s< t$, are aligned with an eigenvector
of $\bV_{-i}$ and therefore also with an eigenvector of $\bV_t$.
Consequently, the eigenvectors of $\sum_{s=1}^{t-1} \bx_s \bx_s\transpose$
are also aligned with $\bV_t$. Since $\bLambda = \bV_t - \sum_{s=1}^{t-1} \bx_s
\bx_s\transpose$
and $\bLambda$ is diagonal, we conclude that $\bV_t$ is diagonal and
all $\bx_s$ are  aligned with the canonical axes.


Now in the case of eigenvalue multiplicities, the maximum of $|\bV_t|$
may be reached by several sets of vectors $\{(\bx^m_s)_{1\leq s< t}\}_m$
but for some $m^*$, the set $(\bx^{m^*}_s)_{1\leq s< t}$
will be aligned with the axes.
In order to see that, consider a perturbed matrix $\bV_{-i}^\varepsilon$
by a random perturbation of amplitude at most $\varepsilon$,
i.e.~such that $\bV_{-i}^\varepsilon\rightarrow \bV_{-i}$ when
$\varepsilon\rightarrow 0$.
Since the perturbation is random, then the probability that
$\bLambda^\varepsilon$,
as well as all other $\bV_{-i}^\varepsilon$ possess an eigenvalue of
multiplicity bigger than 1 is zero.
Since the mapping $\varepsilon\mapsto \bV_{-i}^\varepsilon$  is continuous,
we deduce that any adherent point $\bar \bx_i$ of the sequence
$(\bx_i^\varepsilon)_{\varepsilon}$
(there exists at least one since the sequence is bounded in $\ell_2$-norm)
is aligned with the limit $\bV_{-i}$ and we can apply the previous reasoning.
\end{proof}

\setcounter{lemma}{5-1}
\begin{lemma}
For any $T$, let $\bV_{T+1}=\sum_{t=1}^{T} \bx_t \bx_t\transpose + \bLambda$. Then:
$$\log\frac{|\bV_{T+1}|}{| \bLambda  |} \leq \max \sum_{i=1}^N
\log\Big(1+\frac{t_i}{\lambda_i}\Big),$$
where the maximum is taken over all possible positive real
numbers $\{t_1,\dots,t_N\}$, such that $\sum_{i=1}^N t_i = T$.
\end{lemma}
\begin{proof}[Proof of Lemma~\ref{lemma:detgeo}]
We want to bound the determinant $|\bV_{T+1}|$ under the coordinate
constraints $\|\bx_t\|_2\leq 1$. Let:

$$M(\bx_1,\dots,\bx_{T})= \Big| \bLambda + \sum_{t=1}^{T} \bx_t
\bx_t\transpose\Big|$$

From Lemma~\ref{lemma:2} we deduce that the maximum of $M$ is reached when all
$\bx_t$ are aligned with the axes:
\beqan
M&=&\max_{\bx_1,\dots,\bx_{T}; \bx_t\in \{\be_1,\dots, \be_N\}} \Big| \bLambda +
\sum_{t=1}^{T} \bx_t \bx_t\transpose\Big| \\
&=& \max_{t_1,\dots,t_N \mbox{ positive integers}, \sum_{i=1}^N t_i= T} \Big|
\mbox{diag}(\lambda_i + t_i )\Big| \\
&\leq& \max_{t_1,\dots,t_N \mbox{ positive reals}, \sum_{i=1}^N t_i= T}
\prod_{i=1}^N \Big(\lambda_i + t_i \Big),
\eeqan
from which we obtain the result.

%
%
\end{proof}

\setcounter{lemma}{6-1}

\begin{lemma}
Let $d$ be the effective dimension and $T$ be the time horizon of the algorithm. Then: 
$$\log\frac{|\bV_{T+1}|}{| \bLambda  |}\leq 2 d \log\left(1+\frac{T}{\lambda}\right)$$
\end{lemma}

\begin{proof}[Proof of Lemma~\ref{lem:logdetratio}]
Using Lemma~\ref{lemma:detgeo} and Definition~\ref{def:effectived}:
\begin{align*}
\log\frac{|\bV_{T+1}|}{| \bLambda  |} &\leq \sum_{i=1}^d
\log\left(1+\frac{T}{\lambda}\right) + \sum_{i=d+1}^N
\log\left(1+\frac{t_i}{\lambda_{d+1}}\right) \\
&\leq d \log\left(1+\frac{T}{\lambda}\right) + \sum_{i=1}^N
\frac{t_i}{\lambda_{d+1}}\\
&\leq d \log\left(1+\frac{T}{\lambda}\right) + \frac{T}{\lambda_{d+1}} \leq 2 d \log\left(1+\frac{T}{\lambda}\right)
\end{align*}
\end{proof}

\newpage
\textbf{\SpectralEliminator\/}

\setcounter{lemma}{7-1}
\begin{lemma}
For any fixed $\bx\in\R^N$, any $\delta>0$, and
 $\beta(\delta)=2R \sqrt{14\log(2/\delta)} + \|\balpha^*\|_\bLambda  $, we have 
for any phase~$j$:
$$ \P\Big( |\bx\transpose (\hat\balpha_j-\balpha^*)| \leq \|\bx\|_{\overline{\bV}_j^{-1}}
\beta(\delta) \Big) \geq 1-\delta$$
\end{lemma}

\begin{proof}[Proof of Lemma~\ref{lem:3}]
Defining $\bxi_j=\sum_{s=t_{j-1}}^{t_{j}-1}\bx_s \varepsilon_s$, we have:
\begin{align}
\label{eq:1}
\left|\bx\transpose (\hat \balpha_j-\balpha^*)\right| &= |\bx\transpose (-\overline{\bV}_j^{-1}
\bLambda \balpha^* + \overline{\bV}_j^{-1} \bxi_j )| \nonumber\\
&\leq |\bx\transpose \overline{\bV}_j^{-1} \bLambda \balpha^* | + | \bx\transpose
\overline{\bV}_j^{-1} \bxi_j |
\end{align}
The first term in the right hand side~of \eqref{eq:1} is bounded as
\beqan
|\bx\transpose \overline{\bV}_j^{-1} \bLambda \balpha^*|
&\leq& \|\bx\transpose \overline{\bV}_j^{-1} \bLambda^{1/2}\|  \|\bLambda^{1/2}\balpha^*\| \\
&=& \|\balpha^*\|_\bLambda  \sqrt{\bx\transpose \overline{\bV}_j^{-1} \bLambda \overline{\bV}_j^{-1}
\bx} \\
&\leq & \|\balpha^*\|_\bLambda   \sqrt{\bx\transpose \overline{\bV}_j^{-1} \bx} =
\|\balpha^*\|_\bLambda \|\bx\|_{\overline{\bV}_j^{-1}},
\eeqan
where we used Lemma~\ref{lem:updateInverse} in the second inequality.

Now consider the second term in the r.h.s.~of \eqref{eq:1}. We have:
$$ \left|\bx\transpose \overline{\bV}_j^{-1} \bxi_j\right| = \left|
\sum_{s=t_{j-1}}^{t_{j}-1} (\bx\transpose \overline{\bV}_j^{-1}
\bx_s)\varepsilon_s\right|$$

Let us notice that the points $(\bx_s)_{t_{j-1}\leq s< t_{j}}$selected by the algorithm during phase
$j-1$ only depend on their width $\|\bx\|_{\overline{\bV}_{s}^{-1}}$ which does not depend
on the rewards received during the phase $j-1$. Thus, given $\cF_{j-2}$, the
sequence
$(\bx_s)_{t_{j-1}\leq s< t_{j}}$ is deterministic. Consequently, we can use a
variant
of Azuma's inequality \cite{variant2011shamir}:
\begin{align*}
\P\Bigg(& |\bx\transpose \overline{\bV}_j^{-1} \bxi_j|^2 \leq 28 R^2 2 \log(2/\delta)
\times
\\ & \times  \bx\transpose \overline{\bV}_j^{-1}
\Big(\sum_{s=t_{j-1}}^{t_{j}-1} \bx_s \bx_s\transpose \Big)
\overline{\bV}_j^{-1} \bx \Big|
\cF_{j-2}\Bigg) \geq 1-\delta,
\end{align*}
from which we deduce
$$
\P\Big( |\bx\transpose \overline{\bV}_j^{-1} \bxi_j|^2 \leq 56 R^2 \bx\transpose \overline{\bV}_j^{-1}
\bx  \log(2/\delta) \Big| \cF_{j-2}\Big) \geq 1-\delta,$$
since $\sum_{s=t_{j-1}}^{t_{j}-1} \bx_s \bx_s\transpose \prec \overline{\bV}_j$
by Lemma~\ref{lem:updateInverse}.
Thus:
$$
\P\Big( |\bx\transpose \overline{\bV}_j^{-1} \bxi_j| \leq 2 R \|\bx\|_{\overline{\bV}_j^{-1}}\sqrt{ 14
\log(2/\delta)}\Big) \geq 1-\delta$$
We obtain the statement of the lemma by combining the bounds of the two terms in~\eqref{eq:1} and
setting $\beta(\delta)$  as:
$$\beta(\delta)=2R \sqrt{14\log(2/\delta)} + \|\balpha^*\|_\bLambda$$t
\end{proof}

\setcounter{lemma}{8-1}
\begin{lemma}For all $\bx\in A_j$, $j>1$, we have:
$$\min\left(1,\,\|\bx\|_{\overline{\bV}_{j}^{-1}}\right) \leq \frac{1}{t_{j}-t_{j-1}} \sum_{s=t_{j-1}}^{t_{j}-1}
\min\left(1,\,\|\bx_s\|_{\bV_{s}^{-1}}\right)$$
\end{lemma}

\begin{proof}[Proof of Lemma~\ref{lem:4}]
Using Lemma \ref{lem:updateInverse}, we have:
\begin{align*}
(t_{j}-&t_{j-1}) \min\left(1,\,\|\bx\|_{\overline{\bV}_{j}^{-1}}\right) \\
& \leq  \max_{\bx\in A_j} \sum_{s=t_{j-1}}^{t_{j}-1} \min\left(1,\,\|\bx\|_{\bV_{s}^{-1}}\right)\\
& \leq  \max_{\bx\in A_{j-1}} \sum_{s=t_{j-1}}^{t_{j}-1} \min\left(1,\,\|\bx\|_{\bV_{s}^{-1}}\right)\\
&\leq \sum_{s=t_{j-1}}^{t_{j}-1} \min\left(1,\,\max_{\bx\in A_{j-1}}  \|\bx\|_{\bV_{s}^{-1}}\right)\\
&= \sum_{s=t_{j-1}}^{t_j-1}  \min\left(1,\,\|\bx_s\|_{\bV_{s}^{-1}}\right),
\end{align*}
since the algorithm selects (during phase $j-1$) the arms with largest width.
\end{proof}

\end{document}